\documentclass[11pt]{article}

\usepackage{amsmath,amssymb,amsfonts,amsthm}
\usepackage{mathtools}
\usepackage{fullpage}
\usepackage{makecell}
\usepackage{enumitem}
\usepackage{hyperref}
\usepackage{graphicx}
\usepackage{xcolor}
\usepackage{natbib}
\usepackage{microtype}
\usepackage{booktabs}
\usepackage{mathtools}
\usepackage{authblk}

\usepackage{pgfplots}
\usepackage{tikz}
\pgfplotsset{compat=1.17}

\usepackage{algorithm}
\usepackage{algorithmic}

\DeclareMathOperator{\Var}{Var}
\DeclareMathOperator{\Cov}{Cov}

\newcommand{\E}{\mathbb{E}}
\newcommand{\Pbb}{\mathbb{P}}

\newcommand{\R}{\mathbb{R}}

\newtheorem{theorem}{Theorem}[section]

\newtheorem{corollary}[theorem]{Corollary}
\newtheorem{proposition}[theorem]{Proposition}
\theoremstyle{definition}
\newtheorem{definition}[theorem]{Definition}
\newtheorem{example}[theorem]{Example}

\theoremstyle{remark}
\newtheorem{remark}[theorem]{Remark}

\title{\bf
Bounded Difference Concentration for Infinitely Exchangeable Sequences with Applications to AI Benchmark Uncertainty}

\author[1]{Fangyuan Lin}
\author[2]{Spencer Frei}
\author[1]{Victor H. de la Pe\~{n}a}

\affil[1]{Department of Statistics, Columbia University}
\affil[2]{Google DeepMind}

\affil[ ]{\texttt{fl2744@columbia.edu}, \texttt{sfrei@google.com}, \texttt{vhd1@columbia.edu}}

\date{\today}

\begin{document}

\maketitle

\begin{abstract}
We consider the concentration properties of functions of infinitely exchangeable random variables. By conditioning on the de Finetti directing measure, we show that the deviation of any function with bounded-difference constants $c_1, \dots, c_n$ decomposes into a conditional sampling fluctuation and a latent mixture fluctuation.  When this latent mixture is $\sigma_{\mathrm{mix}}^2$-subgaussian, we establish a concentration inequality with an effective variance proxy of $\frac{1}{4}\sum_i c_i^2 + \sigma_{\mathrm{mix}}^2$. Crucially, we demonstrate that for zero-sum linear contrasts—such as the difference between a subsample mean and a full population mean—the latent mixture term cancels exactly. This cancellation yields a tight, mixture-free Hoeffding-type bound that provides a direct de Finetti mechanism for the infinite-extendibility limit of recent finite-exchangeable concentration results. We apply this framework to quantify uncertainty in composite AI benchmarks, such as MMLU, where question items naturally exhibit exchangeable dependence across domains. Our results provide both a domain-stratified hierarchical model for bounding the uncertainty of accuracy scores, and a distribution-free, cost-saving statistical guarantee for accurately estimating full benchmark scores from random subsets.
\end{abstract}

\medskip
\noindent\textbf{Keywords:} exchangeability, de~Finetti theorem, McDiarmid inequality, bounded differences, concentration inequalities, benchmark uncertainty, MMLU, leaderboard comparisons

\medskip
\noindent\textbf{AMS 2020 Subject Classifications:} Primary 60E15; Secondary 62P30, 62F10.

\tableofcontents

\section{Introduction}\label{sec:intro}
Concentration inequalities are fundamental for deriving generalization bounds and uncertainty guarantees from finite data. For independent random variables, Hoeffding's inequality provides variance-free tails for bounded sums, while McDiarmid's inequality extends these subgaussian tails to any statistic $f(X_1,\dots, X_n)$ that satisfies bounded-difference conditions with constants $c_1, \dots, c_n$ \citep{hoeffding1963probability,mcdiarmid1989method}. However, the classic martingale arguments underlying these results rely on independence. When observations share a latent source of randomness, these standard techniques break down.

There is no single exchangeable analogue of Hoeffding's inequality. Existing literature on variance-free concentration for sample means under exchangeability has largely pursued three directions: concentration for functions of i.i.d. random variables, concentration under structural assumptions on the function, and concentration for weighted sample means of exchangeable random variables around a finite-population mean. For instance, \citet{gottschling2026hoeffding} establish bounds relative to the extreme conditional means in the support of the de Finetti mixing measure, while \citet{foygel2024hoeffding} studies weighted sums centered at the finite-population mean.

In this paper, we develop a bounded-difference concentration inequality for infinitely exchangeable random variables. By de Finetti's theorem \citep{definetti1937}, an infinitely exchangeable sequence is conditionally i.i.d. given a directing random measure $\Theta$. Thus, the classical bounded-difference argument can be applied after conditioning on $\Theta$. Removing this conditioning reveals that the deviation of $f(X)$ from its expectation decomposes into two distinct sources of fluctuation:$$ f(X)-\E f(X) = \underbrace{f(X)-\E[f(X)\mid \Theta]}_{\text{conditional sampling fluctuation}} + \underbrace{\E[f(X)\mid \Theta]-\E f(X)}_{\text{latent mixture fluctuation}}. $$

The first term is controlled by the standard bounded-difference constants; the second term is intrinsic to exchangeability and must be bounded, estimated, or shown to vanish. If the latent mixture fluctuation is $\sigma_{\mathrm{mix}}^2$-subgaussian, we establish the concentration bound
$$ \Pbb\{|f(X)-\E f(X)|\ge t\} \le 2\exp\!\left( -\frac{t^2}{ 2\left(\frac{1}{4}\sum_i c_i^2+\sigma_{\mathrm{mix}}^2\right)} \right). $$
Example~\ref{ex:mixture-tight} shows that the resulting effective proxy is tight at the Gaussian scale: there are infinitely exchangeable Bernoulli laws for which the theorem has the correct exponential rate, up to polynomial factors in the tail parameter.

A key structural observation of this paper is that the mixture term cancels exactly for zero-sum linear contrasts. If $f(X_1,\ldots,X_N)=\sum_{i=1}^N a_iX_i$ with $\sum_{i=1}^N a_i=0$, then conditional identical distribution implies $\E[f(X_1,\ldots,X_N)\mid\Theta] = 0$. Consequently, no estimate of $\sigma_{\mathrm{mix}}^2$ is needed. Applying this observation to a subsample-vs-full sample mean comparison yields a Hoeffding-type bound that is the infinite-extendibility limit of the signed-weight finite-exchangeable result of \citet{foygel2024hoeffding}. Our proof provides the direct de Finetti explanation for this phenomenon: the concentration is a natural consequence of mixture cancellation for centered contrasts.

We apply these theoretical results to uncertainty quantification for composite AI benchmarks, focusing on the Massive Multitask Language Understanding (MMLU) benchmark \citep{hendrycks2020measuring}. A common practice in language model evaluation is to report confidence intervals using standard binomial variance approximations, treating test instances as independent Bernoulli trials \citep{madaan2024quantifying}. For composite benchmarks, this assumption is statistically implausible. Questions from related subjects probe overlapping forms of knowledge; a model that performs well in mathematics is highly likely to perform well in physics. The relevant dependence structure is therefore not independence, but a hierarchical version of exchangeability. This maps perfectly to our theoretical framework, as infinitely exchangeable variables naturally exhibit the positive pairwise correlation inherent to these tasks.

We organize our empirical analysis of six item-level MMLU response datasets around two statistical questions. First, how reliable is the observed MMLU score under a domain-stratified hierarchical exchangeability model? Second, how accurately can the full benchmark score be estimated from a random subset of questions? The first question targets an uncentered statistic, which necessitates a fitted beta--binomial proxy to model the latent subject and domain heterogeneity. The second question targets a subsample-vs-full zero-sum contrast, allowing the de Finetti mixture term to cancel exactly. This provides a direct, distribution-free method to trade evaluation cost against accuracy, demonstrating that benchmark uncertainty depends fundamentally on the estimand.

\section{Related Work}

One of the earliest concentration results for exchangeable data arise in sampling without replacement. \citet{serfling1974probability} proved Hoeffding-type inequalities for the sample mean when $X_1,\dots, X_n$ are drawn without replacement from a fixed population $\{a_1,\dots, a_N\} \subset [a,b]$: for any $t>0$,
\begin{equation}\label{eq:serfling}
  \Pbb\!\left(\bar{X}_n - \mu \ge t\right)
  \le \exp\!\left(-\frac{2nt^2}{(b-a)^2\left(1-\frac{n-1}{N}\right)}\right).
\end{equation}
which is tighter than the bound for independent data by a factor of $(1-(n-1)/N)$. \citet{bardenet2015concentration} subsequently sharpened Serfling's inequality and extended it to a Bernstein analogue that incorporates the empirical variance of the population. A useful survey of this finite-population perspective is \citet{tolstikhin2017concentration}, which reviews Serfling-type improvements which include \citet{el2009transductive} that develops concentration inequalities for more general permutation-symmetric functionals of samples drawn without replacement. Intuitively, sharper bounds than the independent counterparts arise because sampling without replacement induces
{negative} pairwise correlations:
$\Cov(X_i, X_j) = -\sigma^2_{\mathrm{pop}}/(N-1) < 0$, where $\sigma^2_{\mathrm{pop}} = \frac{1}{N} \sum_{i=1}^N(x_i-\mu)^2$ is the population variance. Our infinitely exchangeable setting is different. Infinitely exchangeable sequences exhibit positive pairwise correlations (Proposition~\ref{prop:positive-correlation}). Positive dependence inflates effective variance because positive correlation means the summands tend to move in the same direction, so averaging does not cancel fluctuations as efficiently. 

For general exchangeable sequences, the sample mean need not concentrate around the marginal mean. \citet{gottschling2026hoeffding} therefore prove Hoeffding-style upper and lower tail bounds relative to the largest and smallest conditional means in the support of the de Finetti mixing measure. Let $X_1,\ldots,X_m\in[0,1]$ be infinitely exchangeable, and write $\Theta$ for the de Finetti directing distribution.  Define
\[
    \mu(\Theta)=\E_\Theta X_1,
    \qquad
    \mu=\E_\mu(\Theta),
    \qquad
    \mu_+=\operatorname*{ess\,sup}_{\Theta}\mu(\Theta),
    \qquad
    \mu_- =\operatorname*{ess\,inf}_{\Theta}\mu(\Theta).
\]

Their theorem says that an exchangeable sample mean is contained near the interval $[\mu_-,\mu_+]$. In one-sided form, the result of \citet{gottschling2026hoeffding} gives
\[
    \Pbb\{\bar X_m-\mu_+\ge t\}\le \exp(-2mt^2),
\]
for $0<t<1-\mu_+$ and 

\[
    \Pbb\{\mu_- -\bar X_m\ge t\}\le \exp(-2mt^2),
\]
for $0< t<\mu_-$. A setting where our framework is more informative is when the mixture fluctuation is well behaved. Consider a Bernoulli mixture model $X_i \mid \Theta\sim \Theta$ where the conditional mean $\mu(\Theta)$ is usually near $1/2$ but has rare extreme values. For example, let $\mu(\Theta) = 1/2$ with probability $1-\epsilon$ and $\mu(\Theta) = 0$ or $1$ with probability $\epsilon/2$ each. Then the marginal mean of $X$ is $\mu = 1/2$ while $\mu_- = 0$ and $\mu_+ = 1$. Therefore, the bound in \citet{gottschling2026hoeffding} would give no uncertainty quantification around 1/2. In contrast, the mixture fluctuation $U = \mu(\Theta) - 1/2$ has moment generating function
\[
\mathbb E^{\lambda U} = 1-\epsilon + \epsilon \cosh(\lambda/2)
\]
so its optimal subgaussian proxy tends to zero as $\epsilon\downarrow 0$. With a valid bound on this proxy, our theorem~\ref{thm:mcdiarmid-exch} yields a nontrivial confidence bound around the marginal mean.

\citet{foygel2024hoeffding} controls weighted deviations from the finite-population mean $\bar X_N$. \citet{foygel2024hoeffding} established without parametric assumption that for any exchangeable $X_1,\ldots,X_N \in [-1,1]$ and fixed weights
$w_1,\ldots,w_n$:
\begin{equation}\label{eq:barber-review}
  \Pbb\!\left(\Bigl|\sum_{i=1}^n w_i (X_i - \bar{X}_N)\Bigr| \ge t\right)
  \le 2\exp\!\left(-\frac{t^2}{2\|w\|_2^2\,(1+\varepsilon_N)}\right),
\end{equation}
where $\varepsilon_N = (H_N - 1)/(N - H_N) = O(\log N / N) \to 0$. This result was later extended to richer symmetry structures in \citet{cheng2026concentration} such as matrix-valued exchangeable sequences. This is a natural tool for quantifying subsample-vs-full-sample discrepancies. First, the statistic is centered at the full-sample mean $\bar X_N$, while our main benchmark quantity, the MMLU score, $Z = K^{-1}\sum_k \bar X_k$ is an uncentered statistic, as it is not written as a deviation from a reference mean and for Q1 in Section~\ref{sec:mmlu}, where we ask how reliable the observed MMLU score is, the target is $Z$ itself and this framework does not apply directly. Second, the bound is distribution-free and it cannot exploit the case when $\sigma^2_{\mathrm{mix}}$ is small. Our Theorem~\ref{thm:mcdiarmid-exch} treats general (uncentered) statistics and incorporates $\sigma^2_{\mathrm{mix}}$.  Our Theorem~\ref{thm:subsample-full} treats the special zero-sum subsample--vs--full contrast; for Barber's signed-weight theorem it is the infinite-extension limit, while relative to the direct nonnegative-weight application it gains the finite-population factor $1/(1-n/N)$. 

Recent work has argued that language-model evaluations should be treated as experiments with uncertainty estimates rather than as pure leaderboard point estimates \cite{miller2024adding, madaan2024quantifying}. MMLU is a particularly important example because it aggregates performance across many academic and professional subjects \cite{hendrycks2020measuring}. Dataset-quality work such as MMLU-Redux further shows that benchmark uncertainty also includes annotation and construction error, not only sampling variability \cite{gema2025we}. In this paper, we give concentration results for the dependence structure induced by infinite exchangeability and demonstrate them on item-level MMLU response data.

\section{Setup}\label{sec:setup}
\paragraph{Exchangeability.}
A sequence $(X_i)_{i \ge 1}$ is \emph{infinitely exchangeable} if
$(X_1,\ldots,X_n) \stackrel{d}{=} (X_{\sigma(1)},\ldots,X_{\sigma(n)})$
for every $n\in\mathbb N$ and permutation $\sigma$.

\begin{theorem}[de Finetti \citep{definetti1937}]
\label{thm:definetti}
Let \((X_i)_{i\ge1}\) be an infinitely exchangeable sequence taking values in a standard Borel space.  Then there exists a random probability measure \(\Theta\) such that, conditional on \(\Theta\), the variables \(X_i\) are i.i.d.\ with common law \(\Theta\).
\end{theorem}

Our concentration results presented later (Theorems~\ref{thm:mcdiarmid-exch} and~\ref{thm:subsample-full}) rely on the de~Finetti representation, which requires infinite exchangeability. To apply these bounds in practice, it's worth asking: does the data exhibit a signature that is necessary for infinite exchangeability? The de Finetti representation immediately gives a qualitative check.

\begin{proposition}[Positive correlation under infinite exchangeability \cite{aldous1985exchangeability}]\label{prop:positive-correlation}
Let $(X_i)_{i\ge1}$ be infinitely exchangeable and square-integrable.  Then, for $i\ne j$,
\[
    \Cov(X_i,X_j)\ge0.
\]
Moreover,
\[
    \Cov(X_i,X_j)=\Var(\mu(\Theta)),
    \qquad
    \mu(\Theta)=\E[X_1\mid\Theta].
\]
\end{proposition}

\begin{proof}
By the law of total covariance,
\[
\Cov(X_i,X_j)
=
\E[\Cov(X_i,X_j\mid\Theta)]
+
\Cov(\E[X_i\mid\Theta],\E[X_j\mid\Theta]).
\]
The first term is zero because the variables are conditionally independent given $\Theta$.  The second is $\Var(\mu(\Theta))$.
\end{proof}

Proposition~\ref{prop:positive-correlation} gives a qualitative implication of
the de Finetti representation: shared latent variation induces nonnegative covariance. The positive correlation property has a natural interpretation in the MMLU benchmark setting. Intuitively, a model that excels at abstract algebra is more likely, not less, to perform well on formal logic. The alternative---that the ability is ``traded off'' across domains, so that strength in one area implies weakness in another---would manifest as negative correlations and would violate the infinite exchangeability assumption.

In Section~\ref{sec:mmlu}, we present fixed-model diagnostics based on item-level variance decompositions, intraclass-correlation estimates, and within-domain split-half correlations. We also include a model-centered subject--subject heatmap as a descriptive visualization of residual domain structure across the evaluated models.

\paragraph{Bounded differences.}
A function $f:\mathcal{X}^n \to \R$ satisfies \emph{bounded differences}
with constants $(c_i)_{i=1}^n$ if $|f(x)-f(x')| \le c_i$ whenever $x$ and $x'$
differ only in coordinate $i$. 

For independent $(X_i)$, McDiarmid's inequality \citep{mcdiarmid1989method} guarantees subgaussian concentration around the mean:
\[
    \Pbb(|f(X)-\E f(X)|\ge t)
    \le
    2\exp\!\left(-\frac{2t^2}{\sum_i c_i^2}\right).
\]
This corresponds to an optimal variance proxy of $\frac{1}{4}\sum_i c_i^2$. For comparison, the related Efron--Stein inequality \citep{efron1981jackknife} provides a slightly looser upper bound on the actual variance, $\Var(f) \le \frac{1}{2}\sum_i c_i^2$, by considering the expected squared difference when replacing one coordinate with an independent copy.

The tight proxy $\frac{1}{4}\sum_i c_i^2$ relies fundamentally on the Doob martingale $Z_i = \E[f(X) \mid X_1, \dots, X_i]$. Under independence, observing $X_i$ restricts only the $i$-th argument of $f$, ensuring the martingale differences are bounded by $c_i$. However, under general exchangeability, revealing $X_i$ updates the posterior distribution of the latent measure $\Theta$, which shifts the conditional expectations of all unobserved variables. This global update inflates the martingale differences, rendering the classical variance-free argument invalid and motivating the conditional de Finetti decomposition developed in Section~\ref{sec:main-bound}.

To ground these bounded-difference constants in our motivating application, consider the subject-macro benchmark accuracy $Z$ over $K$ subjects, where $m_k$ is the number of items in subject $k$ and $Z = \frac{1}{K}\sum_{k=1}^K \bar X_k$:
\begin{itemize}[noitemsep]
\item Treated as a function of the $K$ subject averages, the bounded-difference constants are $c_k = 1/K$.
\item Treated as a function of the individual items, the bounded-difference constants are $c_{k,j} = 1/(K m_k)$ for item $j$ in subject $k$.
\end{itemize}

\section{A bounded-difference inequality under de Finetti exchangeability}\label{sec:main-bound}

Let
\[
    U_f
    =
    \E[f(X_1,\ldots,X_n)\mid\Theta]
    -
    \E f(X_1,\ldots,X_n)
\]
be the latent mixture fluctuation. By the tower property, $\E [U_f]=0$.

\begin{definition}[Between-mixture subgaussian proxy]\label{def:sigma-mix}
We say that $U_f$ is $\sigma_{\mathrm{mix}}^2$-subgaussian if
\[
    \E e^{\lambda U_f}
    \le
    \exp\!\left(\frac{\lambda^2\sigma_{\mathrm{mix}}^2}{2}\right)
    \qquad\text{for all }\lambda\in\R.
\]
\end{definition}

This assumption is often reasonable in the settings of interest here. First, if $f$ is bounded, then $\E[f(X)\mid\Theta]$ is also bounded, so $U$ is automatically subgaussian, with a conservative proxy determined by the range of $\E[f(X)\mid\Theta]$. Thus for bounded benchmark scores, such as accuracies in $[0,1]$, the assumption always holds in a trivial sense. In particular, it is satisfied by the beta--Bernoulli hierarchy used in Section~\ref{sec:hierarchical}. The substantive question is whether the corresponding proxy $\sigma^2_{\mathrm{mix}}$ is \emph{small enough} to yield a useful concentration bound. 

On the other hand, the assumption may be loose or uninformative if the latent mixture is highly heterogeneous or strongly multimodal. Thus Definition~\ref{def:sigma-mix} should be viewed as a light-tailed mixture hypothesis: it is mild enough to hold automatically in bounded settings, but its practical value depends on whether $\sigma^2_{\mathrm{mix}}$ is materially smaller than the trivial worst-case bound.

\begin{theorem}[Bounded differences for infinitely exchangeable sequences]\label{thm:mcdiarmid-exch}
Let $(X_i)_{i\ge1}$ be infinitely exchangeable with directing measure $\Theta$.  Let $f:\mathcal X^n\to\R$ satisfy bounded differences with constants $c_1,\ldots,c_n$.  If the directing-measure fluctuation $U_f=\E[f\mid\Theta]-\E f$ is $\sigma_{\mathrm{mix}}^2$-subgaussian, then
\[
    \E\exp\{\lambda(f-\E f)\}
    \le
    \exp\!\left[
        \lambda^2
        \left(
            \frac18\sum_{i=1}^n c_i^2
            +
            \frac12\sigma_{\mathrm{mix}}^2
        \right)
    \right],
\]
and for all $t>0$,
\[
    \Pbb\{|f-\E f|\ge t\}
    \le
    2\exp\!\left[
        -\frac{t^2}{2\sigma_{\mathrm{eff}}^2}
    \right],
    \qquad
    \sigma_{\mathrm{eff}}^2
    =
    \frac14\sum_{i=1}^n c_i^2+\sigma_{\mathrm{mix}}^2.
\]
\end{theorem}

\begin{proof}
Write
\[
    f-\E f
    =
    \underbrace{f-\E[f\mid\Theta]}_{W}
    +
    \underbrace{\E[f\mid\Theta]-\E f}_{U_f}.
\]
Conditional on $\Theta$, the variables are independent and $f$ has the same bounded-difference constants.  Hence the conditional McDiarmid moment bound gives
\[
    \E[e^{\lambda W}\mid\Theta]
    \le
    \exp\!\left(\frac{\lambda^2}{8}\sum_i c_i^2\right).
\]
Therefore
\[
\begin{aligned}
\E e^{\lambda(f-\E f)}
&=
\E\!\left[
    e^{\lambda U_f}\E(e^{\lambda W}\mid\Theta)
\right]\\
&\le
\exp\!\left(\frac{\lambda^2}{8}\sum_i c_i^2\right)
\E e^{\lambda U_f}\\
&\le
\exp\!\left[
        \lambda^2
        \left(
            \frac18\sum_i c_i^2
            +
            \frac12\sigma_{\mathrm{mix}}^2
        \right)
    \right].
\end{aligned}
\]
Chernoff's method yields the upper tail, and applying the same argument to $-f$ gives the two-sided bound.
\end{proof}

\begin{figure}[h!]
    \centering
    \includegraphics[width=0.8\textwidth]{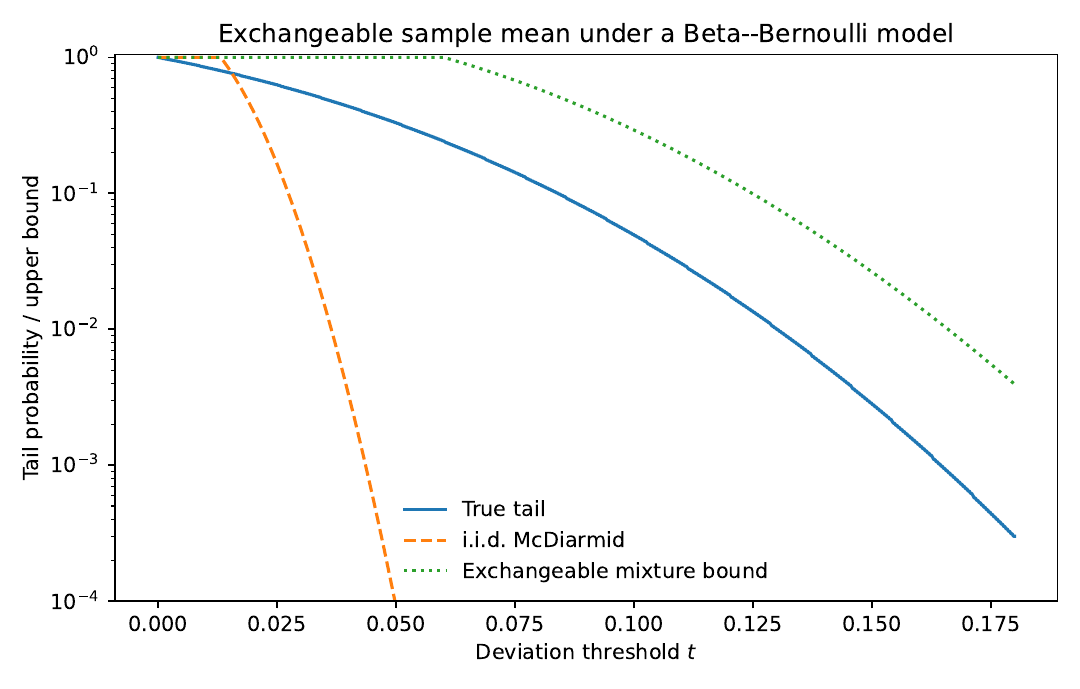}
    \caption{Illustration of Theorem~\ref{thm:mcdiarmid-exch} in a Beta--Bernoulli de Finetti model. The naive independence curve ignores the latent mixture and is therefore not valid.}
    \label{fig:exchangeable-mean-tail}
\end{figure}

Figure~\ref{fig:exchangeable-mean-tail} illustrates the theorem in the Beta--Bernoulli model
\[
    \Theta\sim{\rm Beta}(50,50),
    \qquad
    X_i\mid\Theta\sim{\rm Bernoulli}(\Theta).
\]
in the case of sample mean with $n=2000$.

\begin{example}[Tightness at the Gaussian scale]\label{ex:mixture-tight}
This example shows that the proxy in Theorem~\ref{thm:mcdiarmid-exch} has the correct Gaussian scale.  For each $n$, let $R_1,\ldots,R_n$ be i.i.d. Rademacher random variables and define the de Finetti parameter
\[
    P_n
    =
    \frac12+\frac{1}{4n}\sum_{j=1}^n R_j .
\]
Then $P_n\in[1/4,3/4]$.  Conditional on $P_n$, let
\[
    X_i\mid P_n \stackrel{\mathrm{iid}}{\sim} {\rm Bernoulli}(P_n),
    \qquad i=1,2,\ldots,
\]
so that, for each $n$, the sequence is infinitely exchangeable.  Take
\[
    f_n(X_1,\ldots,X_n)=\bar X_n=\frac1n\sum_{i=1}^n X_i .
\]
The bounded-difference constants are $c_i=1/n$, so the conditional McDiarmid proxy is $\frac14\sum_{i=1}^n c_i^2=\frac{1}{4n}$.
The directing-measure fluctuation is
\[
    U_{f_n}(P_n)
    =
    \E[\bar X_n\mid P_n]-\E\bar X_n
    =
    P_n-\frac12
    =
    \frac{1}{4n}\sum_{j=1}^n R_j .
\]
Since a weighted Rademacher sum with weights $a_j$ is sub-Gaussian with proxy $\sum_j a_j^2$, we may take
\[
    \sigma_{\rm mix}^2
    =
    n\left(\frac{1}{4n}\right)^2
    =
    \frac{1}{16n}.
\]
Therefore Theorem~\ref{thm:mcdiarmid-exch} uses the effective proxy $\sigma_{{\rm eff},n}^2
    =
    \frac{1}{4n}+\frac{1}{16n}
    =
    \frac{5}{16n}.$
On the other hand,
\[
    \sqrt n\, (\bar X_n-1/2)
    =
    \underbrace{\sqrt n\,(P_n-1/2)}_{\to N(0,1/16)}
    +
    \underbrace{\sqrt n\,(\bar X_n-P_n)}_{\to N(0,1/4)} .
\]
The two terms are asymptotically independent: the first is a function of $P_n$
alone, while the second, conditional on $P_n$, is a standardized sum of i.i.d.\
$\mathrm{Bernoulli}(P_n)$ variables, so it depends on the first only through its
conditional variance $P_n(1-P_n)$.  Since $P_n\to\frac12$ and $p\mapsto p(1-p)$
is continuous, this variance converges to $\tfrac14$, so the limiting conditional
law is free of $P_n$, and hence
\[
    \sqrt n\, (\bar X_n-1/2)\Rightarrow N\!\left(0,\tfrac5{16}\right).
\]

A more rigorous argument can be achieved via convergence of characteristic functions.

Consequently, for every fixed $a>0$,
\[
    \Pbb\{\bar X_n-\E\bar X_n\ge a\sigma_{{\rm eff},n}\}
    \longrightarrow
    1-\Phi(a),
\]
whereas Theorem~\ref{thm:mcdiarmid-exch} gives
\[
    \Pbb\{\bar X_n-\E\bar X_n\ge a\sigma_{{\rm eff},n}\}
    \le
    e^{-a^2/2}.
\]
This example shows that the additive proxy $\frac14\sum_i c_i^2+\sigma_{\rm mix}^2$ has the correct asymptotic variance scale in general. In the construction above, the conditional sampling fluctuation contributes $1/(4n)$, the directing-measure
fluctuation contributes $1/(16n)$, and the normalized statistic converges to
a Gaussian with variance equal to the sum of these two contributions. Thus the
mixture proxy cannot, in general, be omitted or replaced by a lower-order term. The resulting tail bound is sharp at the Gaussian-exponent scale.
\end{example}

\section{When the mixture term vanishes and the subsample-vs-full concentration}

For many statistics the mixture proxy must be bounded or estimated.  For zero-sum linear contrasts, it vanishes.

\begin{proposition}[Zero-sum linear contrasts]\label{prop:zero-sum}
Let $(X_i)$ be infinitely exchangeable with $\E|X_1|<\infty$.  For
\[
    f(X_1,\ldots,X_N)=\sum_{i=1}^N a_iX_i,
\]
we have
\[
    \E[f\mid\Theta]
    =
    \left(\sum_{i=1}^N a_i\right)\mu(\Theta),
    \qquad
    \mu(\Theta)=\E[X_1\mid\Theta].
\]
Consequently, if $\sum_i a_i=0$, then $U_f=0$ almost surely and $\sigma_{\mathrm{mix}}^2=0$ is a valid mixture proxy.  Conversely, if $\Var(\mu(\Theta))>0$, then a zero mixture proxy is possible only when $\sum_i a_i=0$.
\end{proposition}

\begin{proof}
The identity follows from conditional identical distribution:
\[
    \E[f\mid\Theta]=\sum_i a_i\E[X_i\mid\Theta]=\left(\sum_i a_i\right)\mu(\Theta).
\]
The variance of the mixture fluctuation is
\[
    v_{\mathrm{mix}}^2(f)
    =
    \left(\sum_i a_i\right)^2\Var(\mu(\Theta)).
\]
Thus zero-sum coefficients make the fluctuation identically zero.  Under nondegeneracy $\Var(\mu(\Theta))>0$, no nonzero coefficient sum can have zero mixture variance, and hence no zero subgaussian proxy.

\end{proof}

The zero-sum calculation is one instance of a more general principle.  The mixture term for a statistic \(f\) is
\[
    U_f=\E[f(X)\mid\Theta]-\E f(X).
\]
Thus mixture cancellation occurs whenever \(\E[f(X)\mid\Theta]\) is constant as a function of the directing measure.  Zero-sum linear contrasts are the most useful case for the benchmark applications in this paper, but the principle is not limited to linear averages. Another conceptual example is an antisymmetric paired statistic. Suppose \(h:\mathcal X^2\to\mathbb R\) is bounded and satisfies \(h(x,y)=-h(y,x)\). Conditional on \(\Theta\), the two arguments in a pair are i.i.d., so
\[
    \E[h(X_1,X_2)\mid\Theta]
    =
    \E[h(X_2,X_1)\mid\Theta]
    =
    -\E[h(X_1,X_2)\mid\Theta].
\]
Hence \(\E[h(X_1,X_2)\mid\Theta]=0\) almost surely.

\subsection{The subsample-vs-full statistic}

For $n\le N$, define
\[
    \bar X_n=\frac1n\sum_{i=1}^n X_i,
    \qquad
    \bar X_N=\frac1N\sum_{i=1}^N X_i,
    \qquad
    T_{n,N}=\bar X_n-\bar X_N.
\]
This can be written as $T_{n,N}=\sum_{j=1}^N a_jX_j$ with
\[
    a_j=
    \begin{cases}
        \frac1n-\frac1N, & j\le n,\\[3pt]
        -\frac1N, & j>n.
    \end{cases}
\]
The coefficients sum to zero, and
\[
    \sum_{j=1}^N a_j^2
    =
    n\left(\frac1n-\frac1N\right)^2
    +(N-n)\frac1{N^2}
    =
    \frac{N-n}{nN}.
\]

\begin{theorem}[Subsample-vs-full-score concentration]\label{thm:subsample-full}
Let $(X_i)_{i\ge1}$ be infinitely exchangeable with $X_i\in[a,b]$ almost surely, and let $R=b-a$.  For $1\le n<N$,
\[
    \Pbb\left\{|\bar X_n-\bar X_N|\ge t\right\}
    \le
    2\exp\!\left(
        -\frac{2t^2 nN}{R^2(N-n)}
    \right).
\]
Equivalently, with probability at least $1-\alpha$,
\[
    |\bar X_n-\bar X_N|
    \le
    R\sqrt{\frac{N-n}{2nN}\log\frac{2}{\alpha}}.
\]
\end{theorem}

\begin{proof}
By Proposition~\ref{prop:zero-sum}, $\E[T_{n,N}\mid\Theta]=0$, so $\sigma_{\mathrm{mix}}^2=0$.  Changing coordinate $j$ changes $T_{n,N}$ by at most $R|a_j|$, so
\[
    \sum_{j=1}^N c_j^2
    =
    R^2\sum_{j=1}^N a_j^2
    =
    R^2\frac{N-n}{nN}.
\]
Apply Theorem~\ref{thm:mcdiarmid-exch}.
\end{proof}

\begin{remark}[Full sample collapse]\label{rem:n-equals-N}
At $n=N$,
\[
    \bar X_n=\bar X_N
\]
deterministically, so the deviation probability is zero for every $t>0$.  The half-width in Theorem~\ref{thm:subsample-full} tends to zero as $n\uparrow N$.
\end{remark}

\begin{remark}[Uniform random subsets]\label{rem:random-subset}
The theorem is stated for the first $n$ coordinates only to simplify notation.  If $S\subset\{1,\ldots,N\}$ is a uniformly random subset of size $n$, chosen independently of the exchangeable sequence, then exchangeability implies
\[
    \frac1n\sum_{i\in S}X_i-\bar X_N
    \stackrel{d}{=}
    \bar X_n-\bar X_N.
\]
Thus the same bound applies to a random subset selected before evaluation.
\end{remark}

\subsection{Connection to prior work}
\label{sec:finite-exch-comparison}

The closest direct comparison for Theorem~\ref{thm:subsample-full} is the finite-vector
weighted-sum bound of \citet{foygel2024hoeffding}. Barber's theorem is stated for finite exchangeable vectors and does not require infinite extendibility. Barber emphasizes that her weighted-sum bounds recover the classical iid constants as the ambient exchangeable length tends to infinity. The de Finetti proof presented in this paper identifies why this happens: the zero-sum structure removes the latent mixture term.

We first state the relevant Hoeffding form of Barber's result in the notation of this paper.

\begin{theorem}[Finite-exchangeable Hoeffding bound \cite{foygel2024hoeffding}]
\label{thm:barber-hoeffding}
Let \(X_1,\ldots,X_N\in[a,b]\) be exchangeable, let
\[
    \bar X_N=\frac1N\sum_{j=1}^N X_j,
    \qquad
    R=b-a,
\]
and let \(w=(w_1,\ldots,w_N)\in\mathbb R^N\) be deterministic.  Define
\[
    H_N=\sum_{k=1}^N\frac1k,
    \qquad
    \varepsilon_N=\frac{H_N-1}{N-H_N}.
\]
Then, for every \(t>0\),
\[
    \Pbb\left\{\left|\sum_{i=1}^N w_i(X_i-\bar X_N)\right|\ge t\right\}
    \le
    2\exp\!\left(
        -\frac{2t^2}{R^2\|w\|_2^2(1+\varepsilon_N)}
    \right).
\]
If \(w_i\ge0\) for all \(i\), then the same inequality holds with
\((1+\varepsilon_N)\) replaced by \(1\):
\[
    \Pbb\left\{
        \left|\sum_{i=1}^N w_i(X_i-\bar X_N)\right|\ge t
    \right\}
    \le
    2\exp\left(
        -\frac{2t^2}{R^2\|w\|_2^2}
    \right).
\]
\end{theorem}

Now let $w_1,\ldots,w_n$ be weights on the first $n$ coordinates and set $w_i=0$ for $i>n$.  Write $W=\sum_{i=1}^n w_i$ and define
\[
    T_{w,N}
    =
    \sum_{i=1}^n w_iX_i
    -
    W\bar X_N
    =
    \sum_{i=1}^N w_i(X_i-\bar X_N).
\]
As a linear statistic in $X_1,\ldots,X_N$, this is a zero-sum contrast with coefficient vector $w_i-W/N$ for $i\le n$ and $-W/N$ for $i>n$.  Its squared norm is
\[
    \sum_{i=1}^n w_i^2-\frac{W^2}{N}.
\]
\begin{corollary}[General zero-mixture weighted bound]\label{cor:weighted}
Under the assumptions of Theorem~\ref{thm:subsample-full},
\[
    \Pbb\{|T_{w,N}|\ge t\}
    \le
    2\exp\!\left[
        -\frac{2t^2}{
        R^2\left(\sum_{i=1}^n w_i^2-W^2/N\right)}
    \right],
\]
whenever the denominator is positive.  If the denominator is zero, then $T_{w,N}=0$ almost surely.
\end{corollary}

\begin{proof}
The proof is identical to Theorem~\ref{thm:subsample-full}.  The coefficient of $X_i$ is $w_i-W/N$ for $i\le n$ and $-W/N$ for $i>n$.  Therefore
\[
\begin{aligned}
\sum_{i=1}^n (w_i-W/N)^2+(N-n)(W/N)^2
&=
\sum_{i=1}^n w_i^2
-\frac{2W}{N}\sum_{i=1}^n w_i
+
\frac{nW^2}{N^2}
+
\frac{(N-n)W^2}{N^2}\\
&=
\sum_{i=1}^n w_i^2-\frac{W^2}{N}.
\end{aligned}
\]
\end{proof}

For the subsample mean, $w_i=1/n$, so $\|w\|_2^2=1/n$, $W=1$, and the denominator becomes
\[
    \frac1n-\frac1N=\frac{1-n/N}{n}.
\]

For a fixed finite vector of length $N$, applying the arbitrary signed-weight theorem of \citet{foygel2024hoeffding} directly to this zero-sum coefficient vector gives the same denominator multiplied by the finite-exchangeability factor from Theorem~\ref{thm:barber-hoeffding},
\[
    1+\varepsilon_N.
\]
The role of infinite exchangeability is that the observed vector can be viewed as the first $N$ coordinates of an exchangeable sequence of any longer length. Hence the same zero-sum coefficients can be padded with zeros and Barber's finite theorem can be applied to extensions of length $M\ge N$. The statistic is unchanged because the padded coefficients still sum to zero:
\[
    \sum_{i=1}^M a_i(X_i-\bar X_M)=\sum_{i=1}^M a_iX_i=T_{w,N}.
\]
The inflation factor is then $1+\varepsilon_M$, which tends to one as $M\to\infty$. This limiting argument explains why the infinite-extendible zero-sum bound in Figure~\ref{fig:subsample-tail} has the iid Hoeffding constant. Figure~\ref{fig:subsample-tail} shows the zero-sum bound that exploits the assumption of infinite extendibility and the direct finite-exchangeable bound; relative to the direct nonnegative-weight application with weights $w_i=1/n$ for $i\le n$, the zero-sum representation gains the finite-population factor
\[
    \frac{1}{1-n/N}.
\]

At $n=N$ the latter ratio diverges and the subsample-vs-full statistic is exactly zero. The present framework gives a complementary perspective by making the mixture-cancellation mechanism explicit, rather than obtaining it through a limiting argument.

\begin{figure}[h!]
    \centering
    \includegraphics[width=.80\textwidth]{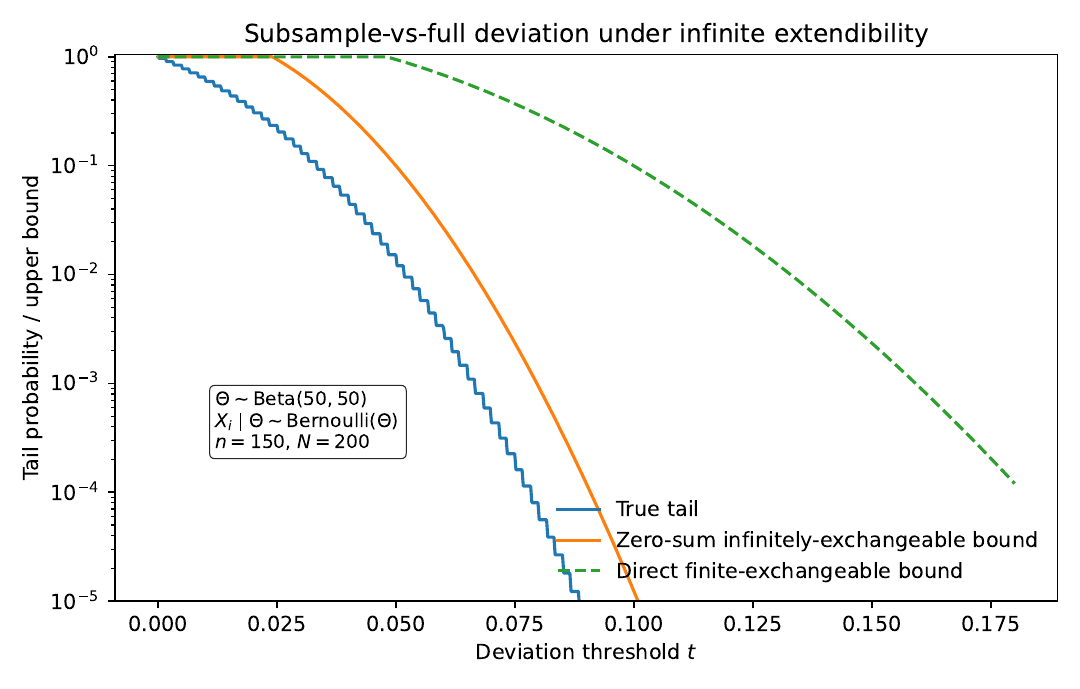}
    \caption{Subsample-vs-full deviation in the same Beta--Bernoulli de Finetti model as Figure~\ref{fig:exchangeable-mean-tail}, with \(n=150\) and \(N=200\). The zero-sum infinitely-exchangeable bound is Theorem~\ref{thm:subsample-full}. The direct finite-exchangeable bound is the application of \citet{foygel2024hoeffding} directly without exploiting the infinite extendibility assumption.}
    \label{fig:subsample-tail}
\end{figure}

\newpage

\section{Hierarchical benchmark scores and variance-proxy estimation}\label{sec:hierarchical}

MMLU and similar benchmarks are not simply exchangeable collections of individual items.  Their questions are organized into subjects, and the subjects are grouped into broader domains. This section introduces the hierarchy used in the empirical analysis and explains how we estimate the \emph{subgaussian variance proxies} that enter the concentration bound in Theorem~\ref{thm:mcdiarmid-exch}.

\subsection{A domain-aware benchmark hierarchy}

Let \(d=1,\ldots,D\) index high-level domains, such as STEM, humanities, social sciences, and other professional subjects. Within domain \(d\), let \(k=1,\ldots,K_d\) index subjects, and within subject \((d,k)\), let \(j=1,\ldots,m_{d,k}\) index questions. The binary variable \(X_{d,k,j}\in\{0,1\}\) records whether a fixed model answers question \((d,k,j)\) correctly.

The working hierarchy has an item layer and a subject layer.  Conditional on the subject-level success probability \(p_{d,k}\), item responses inside subject \((d,k)\) are i.i.d. Bernoulli:
\[
    X_{d,k,j}\mid p_{d,k}
    \stackrel{\mathrm{iid}}{\sim}\mathrm{Bernoulli}(p_{d,k}).
\]
Within a fixed domain \(d\), the subject-level probabilities are modeled as exchangeable draws directed by a domain-level factor \(\Psi_d\):
\[
    p_{d,k}\mid \Psi_d
    \stackrel{\mathrm{iid}}{\sim}G_{d,\Psi_d}.
\]
For the empirical MMLU analysis, the four domains are treated as fixed strata. This avoids estimating a top-level domain mixture from only four observed domains. A model that imposes a general academic ability factor above the domains can still be handled by adding the corresponding top-level mixture proxy.

The theory is easiest to state for a general weighted benchmark score.  Let
\[
    \bar X_{d,k}=\frac1{m_{d,k}}\sum_{j=1}^{m_{d,k}}X_{d,k,j},
    \qquad
    Z_a=\sum_{d=1}^D\sum_{k=1}^{K_d} a_{d,k}\bar X_{d,k},
    \qquad
    a_{d,k}\ge0,
    \quad
    \sum_{d,k}a_{d,k}=1.
\]
The subject-macro MMLU score (the average of the subject-level accuracies) corresponds to \(a_{d,k}=1/K\), where \(K=\sum_dK_d\). The pooled item score (the average over all benchmark questions) corresponds to \(a_{d,k}=m_{d,k}/N\), where \(N=\sum_{d,k}m_{d,k}\). It is called the pooled score because all individual questions are pooled together into one large collection before averaging. The same theorem covers both scores.

\begin{corollary}[Hierarchical bounded-difference bound]
\label{cor:fixed-stratum-hierarchical}
Consider the fixed-domain hierarchy
\[
    \Psi_d,
    \qquad
    p_{d,k}\mid\Psi_d\stackrel{\mathrm{ind}}{\sim}G_{d,\Psi_d},
    \qquad
    X_{d,k,j}\mid p_{d,k}\stackrel{\mathrm{ind}}{\sim}{\rm Bernoulli}(p_{d,k}).
\]
Let
\[
    Z_a=\sum_{d,k}a_{d,k}\bar X_{d,k},
    \qquad
    \bar X_{d,k}=\frac1{m_{d,k}}\sum_{j=1}^{m_{d,k}}X_{d,k,j}.
\]
Write
\[
    m_d(\Psi_d)=\E[p_{d,k}\mid\Psi_d].
\]
Assume that, for each domain \(d\),
\[
    \E\!\left[
        \exp\left\{
            \lambda\bigl(p_{d,k}-m_d(\Psi_d)\bigr)
        \right\}
        \,\middle|\,\Psi_d
    \right]
    \le
    \exp\!\left(\frac{\lambda^2 s_d^2}{2}\right),
    \qquad \lambda\in\R,
\]
for a subject-level proxy \(s_d^2\). Also assume that the fixed-stratum domain fluctuation
\[
    U_{\rm dom}
    =
    \sum_{d=1}^D\sum_{k=1}^{K_d} a_{d,k}m_d(\Psi_d)
    -
    \E\!\left[
        \sum_{d=1}^D\sum_{k=1}^{K_d} a_{d,k}m_d(\Psi_d)
    \right]
\]
is \(\sigma_{\rm dom}^2\)-subgaussian. Then
\[
    \Pbb\{|Z_a-\E Z_a|\ge t\}
    \le
    2\exp\!\left(
        -\frac{t^2}{2\sigma_{\rm hier}^2}
    \right),
\]
where
\[
    \sigma_{\rm hier}^2
    =
    \underbrace{
    \frac14
    \sum_{d,k,j}
    \left(\frac{a_{d,k}}{m_{d,k}}\right)^2
    }_{\text{item-level bounded-difference proxy}}
    +
    \underbrace{
    \sum_{d=1}^D
    s_d^2\sum_{k=1}^{K_d}a_{d,k}^2
    }_{\text{within-domain subject proxy}}
    +
    \underbrace{\sigma_{\rm dom}^2}_{\text{domain mixture proxy}}.
\]
\end{corollary}

\begin{proof}
Decompose
\[
    Z_a-\E Z_a
    =
    \{Z_a-\E[Z_a\mid p]\}
    +
    \{\E[Z_a\mid p]-\E[Z_a\mid\Psi]\}
    +
    \{\E[Z_a\mid\Psi]-\E Z_a\}.
\]
Conditional on \(p=\{p_{d,k}\}\), the item responses are independent.  Changing item \(X_{d,k,j}\) changes \(Z_a\) by at most \(a_{d,k}/m_{d,k}\), which gives the first proxy term.  Next,
\[
    \E[Z_a\mid p]-\E[Z_a\mid\Psi]
    =
    \sum_{d=1}^D\sum_{k=1}^{K_d}
    a_{d,k}\bigl(p_{d,k}-m_d(\Psi_d)\bigr).
\]
Conditional on \(\Psi\), these subject-level terms are independent within each domain, so their subgaussian proxies add and give \(\sum_d s_d^2\sum_k a_{d,k}^2\).  The last term has proxy \(\sigma_{\rm dom}^2\) by assumption.  Multiplying the three conditional moment bounds and applying Chernoff's method proves the claim.
\end{proof}

\subsection{Estimating a subject-level subgaussian proxy}

For a fixed model and fixed domain \(d\), let
\[
    Y_{d,k}=\sum_{j=1}^{m_{d,k}}X_{d,k,j}
\]
be the number of correct answers in subject \(k\).  The beta--binomial working model assumes
\[
    p_{d,k}\stackrel{\mathrm{iid}}{\sim}\mathrm{Beta}(\alpha_d,\beta_d),
    \qquad
    Y_{d,k}\mid p_{d,k}\sim\mathrm{Binomial}(m_{d,k},p_{d,k}).
\]
After integrating out \(p_{d,k}\), the marginal likelihood is
\[
    L_d(\alpha_d,\beta_d)
    =
    \prod_{k=1}^{K_d}
    {m_{d,k}\choose Y_{d,k}}
    \frac{B(Y_{d,k}+\alpha_d,m_{d,k}-Y_{d,k}+\beta_d)}{B(\alpha_d,\beta_d)},
\]
and we compute the maximum-likelihood fit
\[
    (\hat\alpha_d,\hat\beta_d)
    =
    \arg\max_{\alpha_d,\beta_d>0}L_d(\alpha_d,\beta_d).
\]

The concentration bound needs a subgaussian proxy for \(p_{d,k}-\E p_{d,k}\), not merely the ordinary Beta variance.  For \(P_{\alpha,\beta}\sim\mathrm{Beta}(\alpha,\beta)\), the variance is
\[
    v_{\rm Beta}^2(\alpha,\beta)
    =
    \frac{\alpha\beta}{(\alpha+\beta)^2(\alpha+\beta+1)}.
\]
This number is always a lower bound on any valid subgaussian proxy, but it is generally not itself a valid proxy.  Marchal and Arbel show that the centered Beta distribution is strictly subgaussian, in the sense that the optimal proxy equals the variance, only in the symmetric case \(\alpha=\beta\) \citep{marchal2017sub}.

The empirical analysis uses the closed-form Marchal--Arbel upper bound
\begin{equation}\label{eq:marchal-arbel-bound}
    \E\exp\!\left\{\lambda\bigl(P_{\alpha,\beta}-\E P_{\alpha,\beta}\bigr)\right\}
    \le
    \exp\!\left(
        \frac{\lambda^2}{2}\cdot
        \frac{1}{4(\alpha+\beta+1)}
    \right),
    \qquad \lambda\in\mathbb R,
\end{equation}
and plugs in
\begin{equation}
    \widehat s_d^2
    =
    \frac{1}{4(\hat\alpha_d+\hat\beta_d+1)}.
\end{equation}

The plug-in step has the usual model-based asymptotic interpretation.  Suppose, for a fixed domain \(d\), that the beta--binomial model is correctly specified, the true parameter \(\theta_{0,d}=(\alpha_{0,d},\beta_{0,d})\) lies in the interior of \((0,\infty)^2\), and the independent subject-level observations satisfy the standard regularity and nonsingular Fisher-information conditions for maximum likelihood estimation with nonidentically distributed observations.  Then the standard result from MLE theory gives
\[
    \sqrt{K_d}\{(\hat\alpha_d,\hat\beta_d)-\theta_{0,d}\}
    \Rightarrow
    N(0,I_d(\theta_{0,d})^{-1}),
\]
where \(I_d(\theta_{0,d})\) is the limiting per-subject Fisher information \citet{vandervaart1998asymptotic}.  Since
\[
    h(\alpha,\beta)=\frac{1}{4(\alpha+\beta+1)}
\]
is smooth on \((0,\infty)^2\), the delta method gives
\[
    \widehat s_d^2
    =
    h(\theta_{0,d})+O_{\mathbb P}(K_d^{-1/2}).
\]
Consequently the contribution \(\widehat s_d^2/K_d\) to a within-domain macro-average proxy differs from the corresponding oracle plug-in contribution by \(O_{\mathbb P}(K_d^{-3/2})\). Corollary~\ref{cor:fixed-stratum-hierarchical} remains a finite-sample concentration statement when valid deterministic proxies are supplied; the fitted MMLU half-widths in Section~\ref{sec:mmlu} are model-based plug-in summaries. The plug-in step introduces additional estimation variability not reflected in the reported intervals. At the same time, the formula~\ref{eq:marchal-arbel-bound} we use for the beta-Bernoulli proxy is itself an upper bound on the relevant subgaussian proxy, which makes this part of the procedure conservative relative to the fitted model.

\section{MMLU response-data analysis}\label{sec:mmlu}

\subsection{Experimental Setup and Data Collection}
To evaluate the hierarchical exchangeability model and the subsample-vs-full concentration bounds, we collect item-level correctness data on the Massive Multitask Language Understanding (MMLU) benchmark \citep{hendrycks2020measuring}. Specifically, we evaluate the Gemma 2 2B, Gemma 2 9B~\citep{gemma2_2024}, Gemma 3 4B~\citep{gemma3_2025}, Qwen3 0.6B, Qwen3 1.7B, and Qwen3 8B~\cite{qwen3_2025} models using \textit{Simply} \citep{Liang2025Simply}, a JAX-based research codebase for transformers developed by Google DeepMind.
We run a 5-shot multiple-choice evaluation across all 14,042 questions in the MMLU test set. For each question, the model's next-token probabilities for the option letters (A, B, C, D) are computed, and the option with the highest probability is treated as the prediction.  The analysis uses four fixed domains: STEM, Humanities, Social Science, and Other.

\subsection{Parsed scores and subject heterogeneity}

Table~\ref{tab:mmlu-scores} reports both pooled item accuracy and subject-macro accuracy. Recall the pooled score is the average over all $14{,}042$ items.  The macro score is
\[
    Z=\frac1K\sum_{k=1}^K \bar X_k,
    \qquad K=57,
\]
the average of the subject-level accuracies. 

\begin{table}[h!]
\centering
\caption{Parsed MMLU scores from the uploaded response data. }
\label{tab:mmlu-scores}
\begin{tabular}{lrrrr}
\toprule
Model & Pooled accuracy & Macro accuracy & Subject ``SD'' & Subject range \\
\midrule
Qwen3 0.6B & 45.09\% & 46.53\% & 9.72\% & 24.6--64.6\% \\
Gemma 2 2B & 53.23\% & 54.61\% & 15.01\% & 22.1--81.2\% \\
Gemma 3 4B & 58.23\% & 59.54\% & 16.89\% & 26.5--87.2\% \\
Qwen3 1.7B & 58.25\% & 60.27\% & 11.67\% & 33.6--82.9\% \\
Gemma 2 9B & 71.42\% & 72.54\% & 15.52\% & 39.0--95.3\% \\
Qwen3 8B & 74.36\% & 76.44\% & 12.28\% & 45.0--93.8\% \\
\bottomrule
\end{tabular}
\end{table}
Here the subject ``standard deviation'' is reported only as a descriptive empirical fluctuation measure: it uses the usual i.i.d.-style sample variance across subjects, even though the motivating premise of this paper is precisely that subject accuracies may be exchangeably dependent rather than independent. The subject ``SDs'' are large, which is the first empirical sign that treating all questions as independent and homogeneous is misleading: model performance varies substantially by subject.

\subsection{Infinite exchangeability diagnostics}
\label{sec:mmlu-exchangeability-diagnostics}

In the MMLU analysis, we apply Corollary~\ref{cor:fixed-stratum-hierarchical}, which assumes a domain-aware de Finetti hierarchy described in Section~\ref{sec:hierarchical}: within each fixed domain, subjects are modeled as an infinitely exchangeable sequence, and within each subject, questions are modeled as conditionally independent Bernoulli draws given a subject-level latent accuracy.  The purpose of this subsection is to to present diagnostics supporting these modeling assumptions.

The diagnostics below are not formal tests of infinite extendibility. Instead, they check observable implications that would be expected under the hierarchy used in Corollary~\ref{cor:fixed-stratum-hierarchical}. They also help diagnose whether the four-domain stratification is a reasonable working approximation.

For a fixed model, suppress the model index and write the working hierarchy as
\[
    \Psi_d
    \longrightarrow
    p_{d,1},\ldots,p_{d,K_d}\mid \Psi_d
    \longrightarrow
    X_{d,k,1},\ldots,X_{d,k,m_{d,k}}\mid p_{d,k}.
\]
Here \(d\) indexes one of the four fixed MMLU domains, \(k\) indexes a subject inside that domain, and \(j\) indexes a question inside that subject.  The model used in Corollary~\ref{cor:fixed-stratum-hierarchical} is
\[
    p_{d,k}\mid \Psi_d
    \stackrel{\mathrm{iid}}{\sim}
    G_{d,\Psi_d},
    \qquad
    X_{d,k,j}\mid p_{d,k}
    \stackrel{\mathrm{ind}}{\sim}
    \mathrm{Bernoulli}(p_{d,k}).
\]

A useful necessary implication of infinite exchangeability is nonnegative covariance.  Proposition~\ref{prop:positive-correlation} shows that an infinitely exchangeable square-integrable sequence has nonnegative pairwise covariance.  Applied heuristically to the hierarchy above, two fresh questions from the same subject have
\[
    \Cov(X_{d,k,j},X_{d,k,\ell})
    =
    \Var(p_{d,k})
    \ge 0,
    \qquad j\ne \ell,
\]
and two fresh questions from two subjects in the same domain share a nonnegative domain-level covariance contribution. Therefore, strongly negative dependence inside the proposed exchangeability blocks would be evidence against the working hierarchy. Positive dependence does not prove exchangeability, but it is a basic sanity check.

The difficulty is that direct pairwise correlations between fixed questions are not identifiable from a single benchmark administration: for each fixed model, each question is observed only once.  The diagnostics below therefore use quantities that are identifiable from the item-level response matrix.

\paragraph{Split-half reliability within domains.}

The first diagnostic checks whether subject accuracies are stable when each subject is measured using two disjoint random halves of its questions.  For a fixed split repetition \(b\), divide the item set of subject \(k\) into two halves \(A_k^{(b)}\) and \(B_k^{(b)}\), and define
\[
    \hat p_{d,k}^{A,b}
    =
    \frac1{|A_k^{(b)}|}
    \sum_{j\in A_k^{(b)}} X_{d,k,j},
    \qquad
    \hat p_{d,k}^{B,b}
    =
    \frac1{|B_k^{(b)}|}
    \sum_{j\in B_k^{(b)}} X_{d,k,j}.
\]
For each domain \(d\), let \(D_d\) be the set of subjects in that domain and compute
\[
    r_d^{(b)}
    =
    \operatorname{corr}_{k\in D_d}
    \left(
        (\hat p_{d,k}^{A,b})_{k\in D_d},
        (\hat p_{d,k}^{B,b})_{k\in D_d}
    \right),
\]
whenever both vectors have nonzero empirical variance.  We repeat this random splitting procedure 500 times and report the mean split-half correlation for each model and domain.

The intuitive interpretation is simple.  If subject \(k\) has a stable latent accuracy \(p_{d,k}\), then two random halves of its questions should give two noisy estimates of the same quantity.  Subjects that are relatively easy for the model in one half should therefore tend to be relatively easy in the other half.  Large positive split-half correlations indicate that the subject profile is reproducible across item splits.  Values near zero would suggest that the observed subject differences are dominated by item noise.  Negative values would suggest instability or a poor fit to the proposed hierarchy.

The same point can be expressed formally in the idealized infinite-population version of the model.  Conditional on \(p_{d,k}\), two independently drawn half-tests from subject \((d,k)\) have the same conditional mean \(p_{d,k}\) and independent residual item noise.  Therefore, for a randomly selected subject inside domain \(d\),
\[
    \Cov(\hat p_{d,k}^{A},\hat p_{d,k}^{B}\mid \Psi_d)
    =
    \Var(p_{d,k}\mid \Psi_d)
    \ge 0.
\]
This is the empirical feature that the split-half diagnostic is trying to detect.

\begin{table}[h!]
\centering
\caption{Per-model split-half correlations within MMLU domains.  Each entry is the mean, across 500 random splits, of the Pearson correlation between two half-item subject-accuracy vectors restricted to that domain.  The diagnostic checks whether subject-level performance profiles are stable across two independent-looking measurements of the same subjects.}
\label{tab:mmlu-domain-split-half}
\begin{tabular}{lrrrr}
\toprule
Model & STEM & Humanities & Social Science & Other \\
\midrule
Qwen3 0.6B & 0.659 & 0.827 & 0.717 & 0.753 \\
Gemma 2 2B & 0.831 & 0.906 & 0.887 & 0.902 \\
Gemma 3 4B & 0.857 & 0.946 & 0.910 & 0.884 \\
Qwen3 1.7B & 0.772 & 0.885 & 0.802 & 0.825 \\
Gemma 2 9B & 0.889 & 0.941 & 0.897 & 0.893 \\
Qwen3 8B & 0.845 & 0.910 & 0.827 & 0.910 \\
\bottomrule
\end{tabular}
\end{table}

There is one finite-population caveat.  In the idealized argument above, the two half-tests are two fresh samples from the subject-level question population.  In the empirical implementation, the two halves are complementary subsets of the realized finite item set.  Conditional on the realized responses \(x_{k,1},\ldots,x_{k,m_k}\) for a single subject, complementary half means are negatively associated.  If \(|A|=a\), \(B=A^c\), and
\[
    s_k^2
    =
    \frac1{m_k-1}
    \sum_{j=1}^{m_k}
    (x_{k,j}-\bar x_k)^2,
\]
then
\[
    \operatorname{Cov}_{\rm split}
    \left(
        \bar x_{k,A},
        \bar x_{k,B}
        \mid x_{k,1:m_k}
    \right)
    =
    -\frac{s_k^2}{m_k}.
\]
Thus, within one fixed subject, complementary splitting introduces a small negative finite-population effect.  The reported statistic is different: it is the correlation across subjects of the two half-score vectors.  At the covariance level, this combines a positive between-subject component with the negative finite-split component.  Positive values should therefore be interpreted as conservative evidence that stable subject-level heterogeneity is present.

\paragraph{Model-centered subject--subject heatmap.}

The second diagnostic is descriptive rather than a direct check of Theorem~\ref{thm:mcdiarmid-exch}.  It asks whether the chosen MMLU domain stratification captures visible structure in subject performance. 

For each subject \(s\), let \(\hat p_{m,s}\) be the accuracy of model \(m\) on subject \(s\).  The raw subject profile
\[
    v_s
    =
    (\hat p_{1,s},\ldots,\hat p_{6,s})
\]
is dominated by overall model strength: stronger models tend to do better on almost every subject.  To display residual subject structure, we subtract each model's subject-macro average:
\[
    \tilde p_{m,s}
    =
    \hat p_{m,s}
    -
    \frac1{57}
    \sum_{u=1}^{57}\hat p_{m,u},
    \qquad
    \tilde v_s
    =
    (\tilde p_{1,s},\ldots,\tilde p_{6,s}).
\]
The centered heatmap entry for subjects \(s\) and \(t\) is
\[
    r^{\rm cent}_{s,t}
    =
    \operatorname{Corr}_m(\tilde p_{m,s},\tilde p_{m,t})
    =
    \operatorname{Corr}(\tilde v_s,\tilde v_t).
\]
This display asks whether two subjects are relatively easy or difficult together
after each model's overall MMLU level is removed.  Centering is necessary because
stronger models perform better on nearly every subject, so the raw subject
profiles would exhibit trivial positive dependence driven by overall model
strength rather than by any subject-specific relationship.  Subtracting each
model's subject-macro average instead poses the relative question: among models
adjusted for their own overall level, does being comparatively strong on one
subject coincide with being comparatively strong on another?  This is not an
item- or subject-level covariance under the infinite-exchangeability model, but a
descriptive diagnostic computed from only six evaluated models---analogous to
taking six students' math and physics scores, removing each student's overall
ability, and asking whether relative strength in the two subjects is associated.

The following heatmap in Figure~\ref{fig:subject-correlation-heatmap-centered} should not be read as a formal exchangeability test. It uses only six models, and the correlations are across models rather than across repeated benchmark draws for one model. Its role is to check whether the fixed domain partition used in the empirical analysis is qualitatively reasonable. Clear within-domain blocks suggest that subjects in the same coarse domain share residual structure.  Diffuse or contradictory patterns would suggest that the four-domain stratification is too crude.


\begin{remark}[Coarse domain labels]
The MMLU domains are coarse.  For example, college biology and college mathematics both belong to STEM, but they may require different skills: biology relies more on factual scientific knowledge and domain vocabulary, while mathematics relies more on symbolic manipulation and abstract reasoning.  A negative centered association between such subjects is therefore not surprising. It is better interpreted as a caveat about the coarseness of the STEM stratum. In the empirical analysis below, we use the simpler working approximation that subjects are approximately infinitely exchangeable within each domain, directed by a domain-level ability factor.
\end{remark}

\begin{figure}[h!]
    \centering    \includegraphics[width=.86\textwidth]{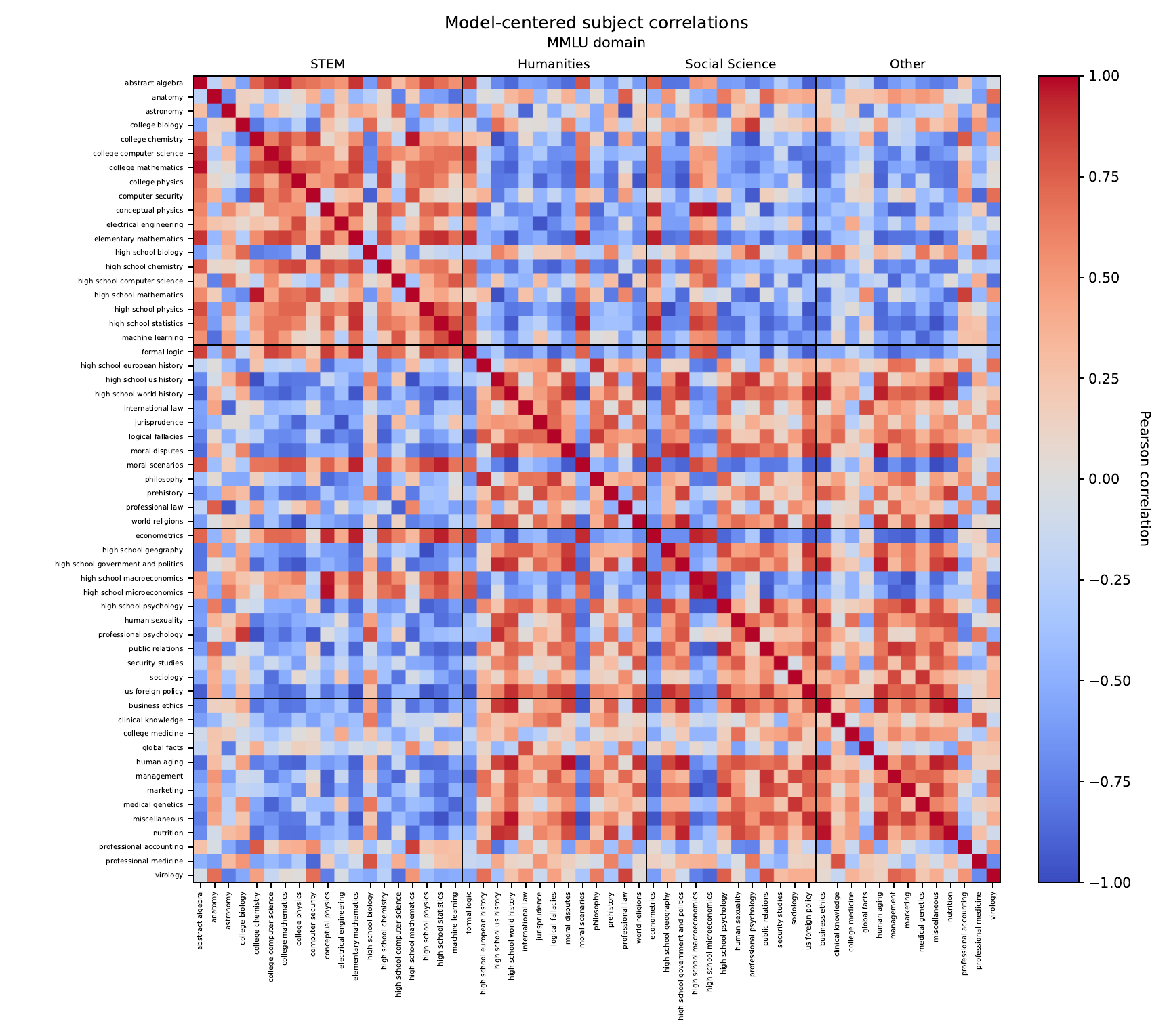}
    \caption{Model-centered, domain-ordered heatmap of pairwise correlations between MMLU subjects.  Each entry is the Pearson correlation between centered subject profiles \(\tilde v_s\) and \(\tilde v_t\), where each model's overall subject-macro accuracy has been removed. The domain labels indicate the fixed strata used in the hierarchical analysis. Red within-domain blocks indicate subjects that are relatively easy or difficult together after removing global model strength. Black indicates near-zero residual association and blue indicates negative residual association.}
    \label{fig:subject-correlation-heatmap-centered}
\end{figure}

In summary, these diagnostics support the empirical use of the fixed-domain hierarchy in Corollary~\ref{cor:fixed-stratum-hierarchical}.  The split-half correlations show that subject-level performance profiles are stable within domains, so the within-domain subject-mixture term is empirically meaningful.  The centered heatmap shows residual domain structure and highlights where the four MMLU domains are useful but coarse.  Neither diagnostic proves infinite exchangeability; rather, they check observable consequences and possible failure modes of the modeling assumptions before the concentration bound is applied.

\subsection{Applying the hierarchical exchangeable bound to the subject-macro score}

We now apply Corollary~\ref{cor:fixed-stratum-hierarchical} to the subject-macro MMLU score.  For this target,
\[
    Z=\frac1K\sum_{k=1}^{K}\bar X_k,
    \qquad K=57,
\]
so the weights in Corollary~\ref{cor:fixed-stratum-hierarchical} are \(a_{d,k}=1/K\).  This is the score for an average subject: each subject receives equal weight, regardless of how many questions it contains. Similarly, the pooled item score is also covered by the same corollary by using \(a_{d,k}=m_{d,k}/N\); we focus on the subject-macro score for notational simplicity.

For the subject-macro score, the item-level bounded-difference proxy is
\[
    \sigma_{\rm item,BD}^2
    =
    \frac{1}{4K^2}\sum_{k=1}^K\frac1{m_k}.
\]
Under the fitted domain-wise beta--binomial working model, let
\[
    \widehat s_{{\rm Beta},m,d}^2
    =
    \frac{1}{4(\hat\alpha_{m,d}+\hat\beta_{m,d}+1)}
\]
denote the Marchal--Arbel subject-level subgaussian proxy for model \(m\) in domain \(d\), as defined in Section~\ref{sec:hierarchical}.  Since there are \(K_d\) equally weighted subjects in domain \(d\), the within-domain subject proxy is
\[
    \widehat\sigma_{{\rm subj}\mid{\rm dom},m}^2
    =
    \sum_{d=1}^D
    \frac{K_d}{K^2}\widehat s_{{\rm Beta},m,d}^2
    =
    \sum_{d=1}^D
    \omega_d^2\frac{\widehat s_{{\rm Beta},m,d}^2}{K_d},
    \qquad
    \omega_d=\frac{K_d}{K}.
\]
The plug-in effective proxy is
\[
    \widehat\sigma_{{\rm eff},m}^2
    =
    \sigma_{\rm item,BD}^2
    +
    \widehat\sigma_{{\rm subj}\mid{\rm dom},m}^2,
\]
and the reported \(1-\alpha\) half-width is
\[
    h_m(\alpha)
    =
    \sqrt{2\widehat\sigma_{{\rm eff},m}^2\log\frac{2}{\alpha}},
    \qquad \alpha=0.05.
\]

Table~\ref{tab:mmlu-theorem-uq} reports this model-based beta-proxy half-width alongside the conventional iid item baseline \(1.96\sqrt{\hat p(1-\hat p)/N}\), which treats the \(14{,}042\) item responses as independent.

\begin{table}[H]
\centering
\caption{95\% uncertainty half-widths for the subject-macro MMLU score. The iid item half-width is the conventional baseline computed as if item responses were independent. The exchangeable beta-proxy half-width applies Corollary~\ref{cor:fixed-stratum-hierarchical} with fitted domain-wise beta--binomial subgaussian proxies.}
\label{tab:mmlu-theorem-uq}
\begin{tabular}{lrrr}
\toprule
Model & Macro score & iid item half-width & \makecell{Exchangeable half-width\\beta proxy} \\
\midrule
Qwen3 0.6B & 46.53\% & 0.82\% & 3.29\% \\
Gemma 2 2B & 54.61\% & 0.83\% & 4.79\% \\
Gemma 3 4B & 59.54\% & 0.82\% & 5.20\% \\
Qwen3 1.7B & 60.27\% & 0.82\% & 4.01\% \\
Gemma 2 9B & 72.54\% & 0.75\% & 5.26\% \\
Qwen3 8B & 76.44\% & 0.72\% & 4.70\% \\
\bottomrule
\end{tabular}

\end{table}

\subsection{Applying the subsample-vs-full theorem to cost-saving evaluation}

We next consider the cost-saving problem.  Suppose a practitioner wants the
full pooled MMLU score
\[
    \bar X_N=\frac1N\sum_{i=1}^N X_i,
    \qquad N=14{,}042,
\]
but evaluates only a uniformly selected subset of \(n\) items.  The estimation
error is
\[
    \bar X_n-\bar X_N.
\]
This is a zero-sum linear contrast, so the de Finetti mixture term cancels
exactly. For binary correctness variables, Theorem~\ref{thm:subsample-full}
gives
\[
    \Pbb\{|\bar X_n-\bar X_N|\ge t\}
    \le
    2\exp\left(-\frac{2t^2nN}{N-n}\right),
    \qquad n<N.
\]
Equivalently, with probability at least \(1-\alpha\),
\[
    |\bar X_n-\bar X_N|
    \le
    h_{n,N}(\alpha)
    :=
    \sqrt{\frac{N-n}{2nN}\log\frac2\alpha}.
\]
This half-width depends only on \(n\), \(N\), and the range of the response
variable. It is therefore common across the six evaluated models.  The
model-specific quantity is the empirical distribution of
\(|\bar X_n-\bar X_N|\) for each realized response vector.

Figure~\ref{fig:subsample-power} displays the resulting cost--accuracy tradeoff.
Panel (a) compares Theorem~\ref{thm:subsample-full} with the direct nonnegative-weight application of the finite-vector exchangeable bound of \citet{foygel2024hoeffding}.  The gain from using the zero-sum representation becomes largest when \(n\) is a non-negligible fraction of \(N\), and the theorem half-width vanishes at \(n=N\), where \(\bar X_n=\bar X_N\) deterministically.  Barber's signed-weight theorem applied to longer finite extensions has the same limiting exponent as Theorem~\ref{thm:subsample-full}.
Panel (b) overlays the theorem half-width with the empirical \(95\%\) errors
computed separately for each of the six realized MMLU response vectors.  The
empirical curves lie below the theorem curve for every model.

\begin{figure}[h!]
    \centering
    \includegraphics[width=\textwidth]{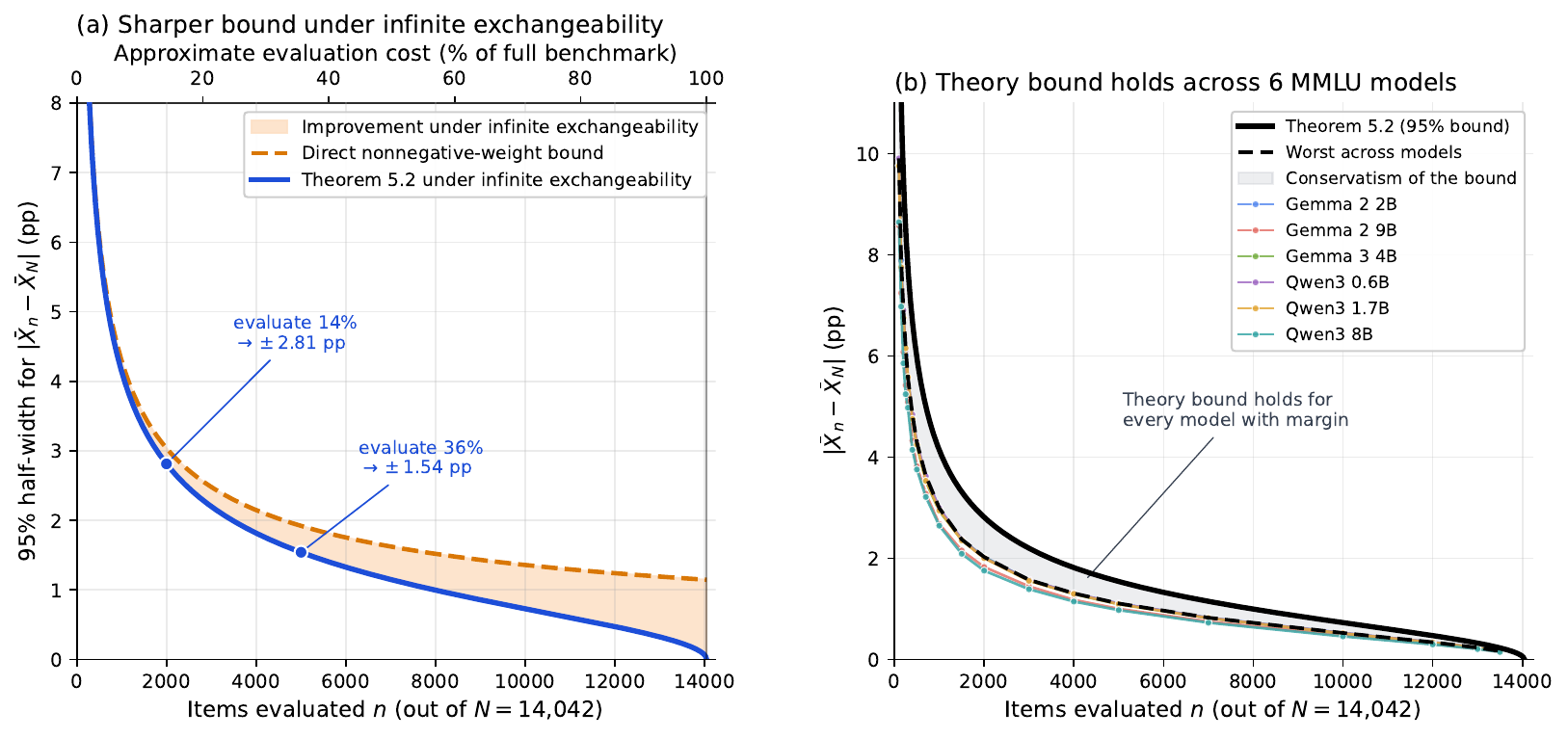}
    \caption{Cost--accuracy tradeoff for estimating the full MMLU pooled score
    from a random subset.  Panel (a) compares the \(95\%\) half-width from
    Theorem~\ref{thm:subsample-full} with the direct nonnegative-weight finite-vector bound
    of \citet{foygel2024hoeffding}.  Panel (b) compares the common theorem
    half-width with the empirical \(95\%\) subsample errors computed separately
    for the six realized model response vectors.}
    \label{fig:subsample-power}
\end{figure}

Table~\ref{tab:mmlu-subsample} reports the same comparison numerically.  For
each \(n\), the theorem half-width is common across models, while the empirical
columns summarize the six model-specific exact finite-subset distributions.

\begin{table}[H]
\centering
\caption{Subsample-vs-full errors on the MMLU response data.  The theorem
half-width is Theorem~\ref{thm:subsample-full} with \(N=14{,}042\) and \(\alpha=0.05\), and is common across models. The empirical columns are computed separately for each realized model response vector and then summarized
across the six models.}
\label{tab:mmlu-subsample}
\resizebox{\textwidth}{!}{\begin{tabular}{rrrrr}
\toprule
\(n\) & Mean empirical 95\% error & \makecell{Worst empirical 95\% error\\across 6 models} & Theorem half-width & \makecell{Max violation rate\\across 6 models} \\
\midrule
250 & 5.85 pp & 6.17 pp & 8.51 pp & 0.65\% \\
500 & 4.10 pp & 4.29 pp & 5.96 pp & 0.70\% \\
1,000 & 2.86 pp & 2.99 pp & 4.14 pp & 0.63\% \\
2,000 & 1.94 pp & 2.03 pp & 2.81 pp & 0.67\% \\
5,000 & 1.06 pp & 1.11 pp & 1.54 pp & 0.65\% \\
10,000 & 0.50 pp & 0.52 pp & 0.73 pp & 0.64\% \\
\bottomrule
\end{tabular}
}
\end{table}



Thus a \(5{,}000\)-item subset gives a model-agnostic
\(1.54\)-percentage-point planning guarantee relative to the full benchmark score while using only \(35.6\%\) of the benchmark. If a full benchmark run
costs \(C_{\rm full}\), then the approximate subset cost under linear
per-item pricing is $0.356\,C_{\rm full}$. This is the cleanest applied payoff of Theorem~\ref{thm:subsample-full}, and the cost-saving guarantee is free of any estimate of
\(\sigma_{\mathrm{mix}}^2\).

\begin{remark}[Paired model comparisons from subsamples]
The same subsample--vs--full idea can also be applied to paired model
comparisons.  Suppose two models \(A\) and \(B\) are evaluated on the same
items, and define the item-level paired difference
\[
    D_i = X_{B,i}-X_{A,i}.
\]
The full paired gap is \(N^{-1}\sum_{i=1}^N D_i\), while the paired gap
estimated from a subset \(S\) of size \(n\) is \(n^{-1}\sum_{i\in S}D_i\).
Their difference,
\[
    \frac1n\sum_{i\in S}D_i
    -
    \frac1N\sum_{i=1}^N D_i,
\]
is again a zero-sum linear contrast. Thus one may use the same concentration principle to quantify how many questions are needed to estimate the full paired model comparison to a
desired accuracy.
\end{remark}

\subsection{Limitations of the empirical analysis}

The empirical analysis should be interpreted with appropriate caution. First, infinite exchangeability assumptions are modeling assumptions and cannot be proved from one finite benchmark realization. The split-half correlations are diagnostics for necessary implications of the hierarchy, not formal tests of infinite exchangeability. Second, the beta--binomial quantities are model-based plug-in proxies. The displayed subject-macro intervals do not account for the additional uncertainty in estimating \(\hat\alpha_d,\hat\beta_d\) from only the observed subjects. Third, MMLU itself contains annotation and dataset-quality issues, as discussed by \citet{gema2025we}.

\section{Conclusion}\label{sec:conclusion}

We developed a bounded-difference concentration framework for infinitely
exchangeable sequences.  The starting point is the de Finetti representation:
after conditioning on the directing measure, the observations are independent,
so the usual bounded-difference argument applies conditionally.  Removing the
conditioning separates the deviation of a statistic into two components: a
conditional sampling fluctuation and a latent mixture fluctuation.  This gives
the effective proxy
\[
    \sigma_{\mathrm{eff}}^2
    =
    \frac14\sum_i c_i^2+\sigma_{\mathrm{mix}}^2,
\]
where the first term is the ordinary bounded-difference contribution and
\(\sigma_{\mathrm{mix}}^2\) controls the between-mixture fluctuation. 

The main structural simplification occurs for zero-sum linear contrasts.  If
\[
    f(X_1,\ldots,X_N)=\sum_{i=1}^N a_iX_i,
    \qquad
    \sum_{i=1}^N a_i=0,
\]
then
\[
    \E[f(X_1,\ldots,X_N)\mid\Theta]=0,
\]
and the mixture term vanishes exactly.  This cancellation yields a clean
subsample-vs-full concentration inequality:
\[
    \Pbb\left\{|\bar X_n-\bar X_N|\ge t\right\}
    \le
    2\exp\!\left(
        -\frac{2t^2 nN}{R^2(N-n)}
    \right),
    \qquad X_i\in[a,b], \quad R=b-a.
\]
Thus, although positive exchangeable dependence generally weakens
concentration for uncentered averages, the discrepancy between a random
subsample score and the full benchmark score remains controllable
because the latent mean cancels.

The MMLU analysis illustrates why this distinction matters.  Treating all
benchmark items as independent Bernoulli trials gives intervals that are too
narrow because it ignores subject-level and domain-level heterogeneity.  The
observed response data show substantial subject variation, positive dependence
patterns, and stable split-half structure, all of which are consistent with a
hierarchical infinite exchangeability interpretation.  For the subject-macro benchmark
score, the uncertainty should therefore include both item-level noise and
latent subject variation. In contrast, for the cost-saving problem of estimating the full benchmark score from a random subset of items, the relevant statistic is a zero-sum contrast, so the subsample-vs-full score bound in the paper applies without estimating a mixture proxy.

The empirical subsample experiments support this practical message. On the realized MMLU response vectors, the theoretical half-widths are conservative but informative, and they provide a direct way to trade off evaluation cost against accuracy relative to the full benchmark score. For example, evaluating only a moderate fraction of the benchmark can already give a several-percentage point guarantee for the full-score approximation. This is an applied payoff of Theorem~\ref{thm:subsample-full}: it gives a planning bound for cheaper benchmark evaluation under an exchangeable-model assumption, while avoiding the hard problem of estimating \(\sigma_{\mathrm{mix}}^2\).

Several limitations remain. First, the empirical diagnostics in this paper check qualitative consequences of the hierarchy, such as positive dependence and stable subject profiles, but they do not prove infinite extendibility. Second, for uncentered benchmark quantities, the usefulness of the bound depends on obtaining a meaningful subgaussian proxy for the latent mixture term. The beta--binomial fits used here provide a model-based and descriptive route to such proxies, but fully nonparametric estimation of \(\sigma_{\mathrm{mix}}^2\) remains an important open problem.

The broader lesson is that benchmark uncertainty depends on the estimand. If the goal is to estimate a fixed benchmark score from a random
subset, the subsample-vs-full concentration gives a sharp and directly usable guarantee. If the goal is to infer a model's broader ability across a population of possible subjects or questions, latent mixture uncertainty is intrinsic and must be modeled or bounded. Exchangeability provides a natural language for separating these two problems. We hope that this decomposition will be useful beyond MMLU, especially for large composite benchmarks where dependence across tasks is a feature of the evaluation rather than a nuisance to be ignored.

\bibliographystyle{abbrvnat}
\bibliography{refs}

\end{document}